%% file: acl_latex.tex
\documentclass[11pt]{article}

\usepackage[preprint]{acl}

\usepackage{times}
\usepackage{latexsym}
\usepackage{amsmath}
\usepackage{enumitem}

\usepackage[T1]{fontenc}

\usepackage[utf8]{inputenc}

\usepackage{microtype}

\usepackage{inconsolata}

\usepackage{graphicx}
\usepackage{booktabs}
\usepackage{tabularx}
\usepackage{subcaption}
\usepackage{cleveref}
\usepackage{amsthm}
\usepackage{amssymb}

\newtheorem{lemma}{Lemma}

\title{Beware of the Batch Size: Hyperparameter Bias in Evaluating LoRA}

\author{Sangyoon Lee, \hspace{1.5em} Jaeho Lee\\
  Pohang University of Science and Technology
 (POSTECH) \\
  \texttt{\{sangyoon.lee, jaeho.lee\}@postech.ac.kr} \\}

\begin{document}
\maketitle
\begin{abstract}
Low-rank adaptation (LoRA) is a standard approach for fine-tuning large language models, yet its many variants report conflicting empirical gains, often on the same benchmarks. We show that these contradictions arise from a single overlooked factor: the batch size. When properly tuned, vanilla LoRA often matches the performance of more complex variants. We further propose a proxy-based, cost-efficient strategy for batch size tuning, revealing the impact of rank, dataset size, and model capacity on the optimal batch size. Our findings elevate batch size from a minor implementation detail to a first-order design parameter, reconciling prior inconsistencies and enabling more reliable evaluations of LoRA variants.

\end{abstract}

\input{latex/1_intro}
\input{latex/2_exp}
\input{latex/2_main}
\input{latex/3_main2}
\input{latex/5_conclusion}
\input{latex/6.limit}

\bibliography{custom}

\newpage
~
\newpage
\appendix
\input{latex/App_related}
\input{latex/App_steps}
\input{latex/App_warmpup}
\input{latex/App_config}
\input{latex/App_lr}
\input{latex/App_theoretical}
\input{latex/App_ablation}
\input{latex/App_gpuhour}

\input{latex/App_concurrent}
\input{latex/App_license}

\end{document}

%% file: latex/1_intro.tex
\section{Introduction}

Low-rank adaptation, or simply LoRA \citep{Hu2022LoRA}, has emerged as the de facto standard method for the parameter-efficient fine-tuning of large language models (LLMs). Building on this success, numerous LoRA variants have been proposed, each claiming a performance gain over the vanilla LoRA \citep{Zhang2023AdaLoRA, Liu2024dora, Meng2024pissa, Wang2025milora, Kalaj2023rslora}.

A disturbing fact about LoRA variants is that they often contradict each other. For example, two recent works---PiSSA and MiLoRA---propose conflicting initialization strategies for LoRA, each focusing on the faithful preservation of the principal singular vectors and minor singular vectors of the pretrained LLM weights, respectively \citep{Meng2024pissa,Wang2025milora}. More intriguingly, experiments suggest that both methods provide performance gains on seemingly identical benchmarks. Why does such a contradiction take place?

In this work, we reveal that this apparent contradiction stems from the inconsistency in a critical yet overlooked hyperparameter: the \textit{\textbf{batch size}}. In particular, we observe that, given a properly tuned batch size, the vanilla LoRA can match or even outperform its more complex variants (\Cref{fig:method}). This finding highlights the pivotal role of a careful batch size selection, not only for effective LLM fine-tuning, but also for a reliable and reproducible evaluation of LoRA variants.

Despite its importance, the selection of LoRA batch sizes primarily relies on crude heuristics, e.g., ``smaller is better'' \citep{schulman2025lora_regret}. While the influence of batch size has been extensively studied for full-parameter training \cite{keskar2017, shallue2019, Zhang2025CBS_llm, Pareja2025unveil_SFT}, these insights do not directly extend to LoRA, operating under unique constraints and demonstrating distinct optimization dynamics \citep{Shuttleworth2025, Biderman2025learnless}. Furthermore, given that LoRA is primarily used in resource-constrained settings, an exhaustive hyperparameter sweep of the batch size may be infeasible as well.



To address this gap, we initiate a systematic study toward understanding how various scale parameters of LoRA workloads---rank, dataset scale, and model size---affect the optimal batch size. 
In particular, we find that the optimal batch size remains relatively consistent under the changes in terms of the rank and the model size, but not for the dataset size. This observation suggests a low-cost proxy upon which one may tune the batch size on small-scale LLMs with low rank, then transfer the batch size to the larger models.

In summary, our contributions are threefold:



\begin{enumerate}
\item We demonstrate that vanilla LoRA remains a highly competitive baseline when batch sizes are properly tuned, suggesting that reported gains in its variants are partially artifacts of suboptimal hyperparameter choices.

\item We find that using a larger batch does not necessarily lead to strictly worse accuracy. Instead, we identify an optimal batch size that maximizes test performance.

\item  We establish a practical guideline for small-scale proxies, showing that the optimal batch configurations can be identified using lower ranks and smaller model capacities as long as dataset scale is preserved.
\end{enumerate}

Taken together, our findings shed light on a critical yet overlooked role of batch size in LoRA. We hope this work encourages more robust and reproducible evaluation practices and provides actionable guidance for practitioners deploying LLMs under real world constraints.

%% file: latex/2_exp.tex
\section{Experimental Setup}
\label{sec:setup}

Our experimental setup primarily follows that of \citet{Meng2024pissa}, with varying selections of the batch sizes. 
All reported figures, except for those in Figures \ref{fig:step}, \ref{fig:warmup}, and \ref{fig:lr}, are an average of three random seeds. Other details are as follows.

\paragraph{Model.} We mainly use the LLaMA-2-7B \cite{Touvron2023llama2} as a base model. This choice is to ensure consistency with the key baselines \citep{Meng2024pissa,Wang2025milora}. We also provide additional experiments with other recent model families, Qwen3-0.6B \cite{yang2025qwen3} and Gemma3-1B \cite{Kamath2025Gemma3}, in Appendix \ref{app:ablation}.


\paragraph{Datasets.} We follow the standard evaluation of the mathematical reasoning task: We fine-tune the model on MetaMathQA \cite{metamath}, and then evaluate on GSM8K benchmark \cite{gsm8k}. By default, we use the first 100K samples in the training split of the MetaMathQA. To demonstrate the generalizability of our findings to other tasks, results on the HumanEval benchmark \cite{humaneval}, fine-tuned on CodeFeedback dataset \cite{codefeedback}, is provided in Appendix \ref{app:ablation}.

\paragraph{Learning rate sweep.} For each batch size selected, we conduct a hyperparameter sweep to find the optimal learning rate. This is mainly due to the prior observations in the full-parameter training literature which indicates that the batch size closely interacts with the learning rate \cite{dontdecaylearningrate}. We select learning rates in the range of $1\times10^{-5}$ to $3\times10^{-3}$; see Appendix \ref{app:opt_lr} for details.

\paragraph{Fixed sample protocol.} To decouple the impact of batch size from the total number of training tokens used, our primary experiments are conducted under a fixed sample protocol restricted to a single epoch. This setup aligns with data efficiency objectives and reflects standard practice in the supervised fine-tuning literature \cite{Pareja2025unveil_SFT}, where limited data availability necessitates precise control over training duration to mitigate overfitting.

Note that this configuration implies that larger batch sizes undergo fewer gradient updates, which is one potential reason why the heuristic of ``smaller is better'' holds. To provide a more comprehensive view, we include results where the number of optimization steps is fixed across the setups (Appendix \ref{app:steps}). In such a case, the total number of processed samples varies across batch size setups. These complementary results indicate that while the confounding effect of steps cannot be entirely discounted, the batch size inherently possesses a critical impact, showing diminishing returns even with increased data volume. Unless otherwise stated, we adopt the fixed sample protocol for our experiments to align with common fine-tuning practices.

\paragraph{Other hyperparameters.} To isolate the effects of batch size, all other hyperparameters are kept identical across configurations (see Appendix \ref{app:config}). Notably, we omit the warm-up phase, as varying batch sizes alters the total step count under our fixed sample protocol (see Appendix \ref{app:warmup}).

%% file: latex/2_main.tex
\section{Re-evaluating LoRA Variants: The Impact of Batch Size}

In this section, we evaluate vanilla LoRA alongside two prominent variants: PiSSA \cite{Meng2024pissa}, which initializes adapters using principal singular vectors, and MiLoRA \cite{Wang2025milora}, which uses minor singular vectors. These opposing studies claim state-of-the-art performance on identical benchmarks, while resulting in contradictory conclusions derived from disparate experimental configurations. To ensure a rigorous comparison, we re-evaluate these methods within a unified framework, as illustrated in Figure \ref{fig:method}.


\input{latex/figure/method}

\subsection{Dissecting reported gains: Methodological superiority or hyperparameter bias?}

From Figure \ref{fig:method}, we observe a ``performance crossover'' between two LoRA variants: PiSSA achieves superior accuracy in the large batch regime, whereas MiLoRA beats PiSSA with smaller batches. Crucially, this discrepancy aligns with the configurations used in their respective original papers, in which MiLoRA was evaluated using smaller batch sizes than PiSSA. This suggests that their reported gains and opposing insights stem not from inherent algorithmic superiority, but from a confounding bias induced by batch size selection. Thus, no single variant remains universally superior across the entire batch size spectrum. Remarkably, our empirical results demonstrate that the selection of batch size alone can induce a substantial performance gap of over 10\% in final accuracy.

Furthermore, we find that vanilla LoRA remains a highly competitive baseline when evaluated under fair conditions. It exhibits comparable performance to its variants across varying batch sizes and achieves the best overall result when its hyperparameters are optimally tuned. This underscores that the critical role of batch size has been largely overlooked in LoRA literature.

\subsection{Challenging prevalent heuristics}

Contrary to conventional heuristics suggesting that smaller batch sizes are universally superior for LoRA fine-tuning \cite{schulman2025lora_regret}, we demonstrate that performance does not scale monotonically as batch size decreases. Instead, there exists an optimal batch size threshold that maximizes hardware efficiency without compromising final accuracy. 

This non-monotonic trend is robust across multiple LoRA variants, model architectures, and tasks while the exact critical threshold varies (see Figure \ref{fig:others}). Given that LoRA is typically deployed in resource constrained environments, maximizing hardware utility is pivotal; however, this must be balanced against systemic variables such as LoRA rank, dataset scale, and model capacity. Section \ref{sec:determinants} provides a rigorous analysis of the interactions between batch size and these determinants to establish systematic guidelines for optimal configuration.


%% file: latex/figure/method.tex
\begin{figure}[h]
\begin{center}
\includegraphics[width=0.42\textwidth]{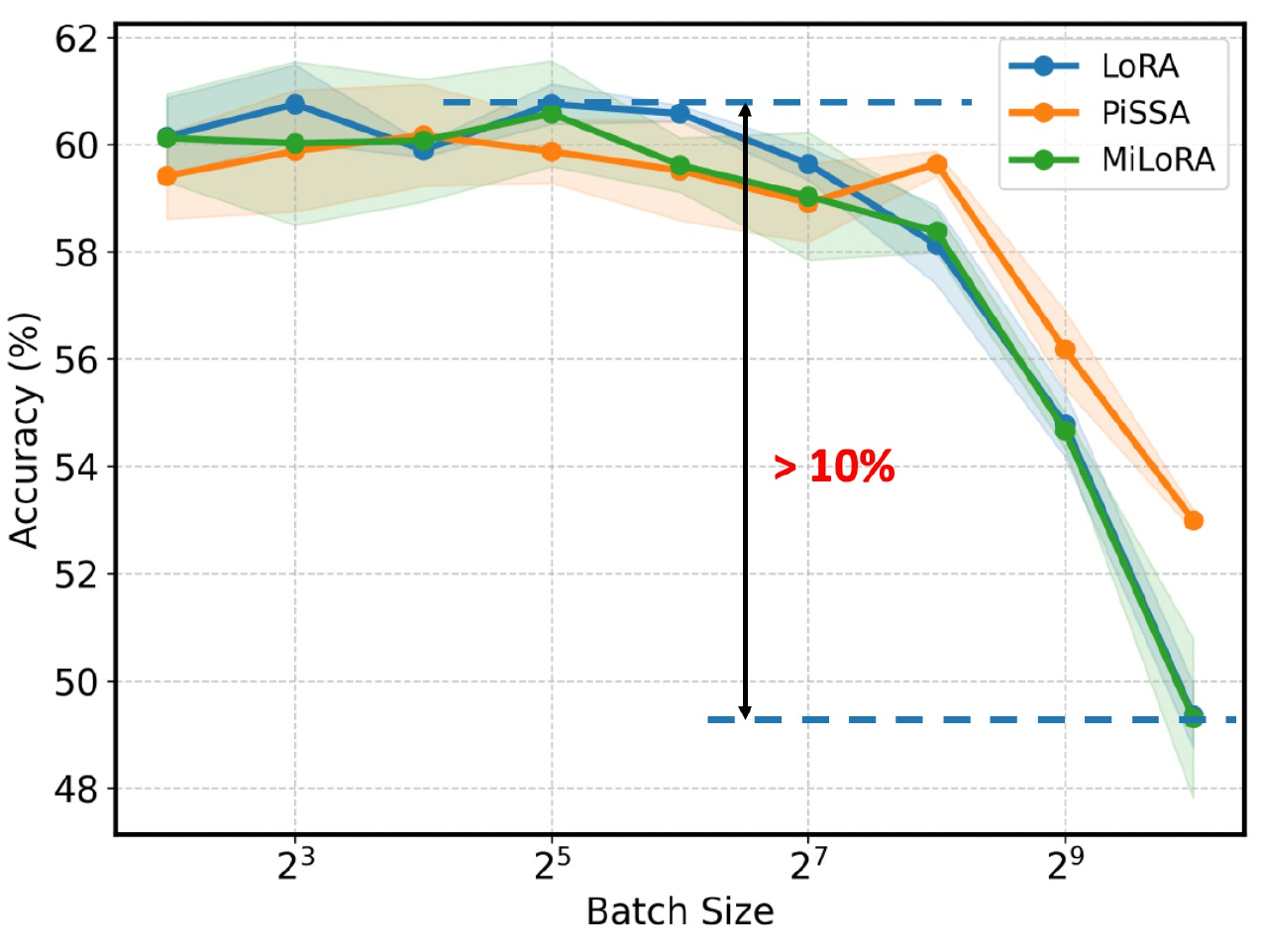}
\end{center}
   \caption{\textbf{Impact of batch size across LoRA variants.} We observe that batch size selection alone can lead to a performance gap of over 10\% in accuracy. Notably, when evaluated at its optimal batch size, vanilla LoRA beats both PiSSA and MiLoRA in math reasoning task.
}
   \label{fig:method}
\end{figure}

%% file: latex/3_main2.tex
\section{Determinants of Batch Size Effect}
\label{sec:determinants}

\input{latex/figure/shifted_acc}

This section identifies the fundamental factors required to construct a reliable low-cost proxy for batch size tuning. We investigate how the impact of batch size varies with three key determinants: (i) LoRA rank, (ii) dataset scale, and (iii) base model capacity. These variables are selected as they represent the primary factors of the total computational overhead in LoRA-based training. For each analysis, we maintain our setup established in Section \ref{sec:setup}, while altering each determinant independently. We reveal a critical insight for establishing a proxy: while batch size effects are robust to changes in LoRA rank and base model size, they are highly sensitive to dataset scale. 

\subsection{Impact of LoRA rank}
We first investigate the interaction between batch size and the rank $r$ of LoRA adapters. Figure \ref{fig:rank_shift} illustrates the test accuracies across a range of batch sizes for ranks spanning from 32 to 256. While minor variations exist, the fundamental performance trend remains consistent across all evaluated ranks. Notably, in Figure \ref{fig:rank_org}, increasing $r$ beyond a certain threshold yields marginal performance gains. This suggests that in some scenarios, employing a moderate LoRA rank in conjunction with a larger batch size can yield a more Pareto optimal configuration for balancing performance and throughput.

\subsection{Impact of dataset scale}
We compare training configurations using subsets of the MetaMathQA dataset ranging from 25K to 200K samples. As illustrated in Figure \ref{fig:data_shift}, LoRA training on larger datasets effectively accommodates larger batch sizes without performance degradation compared to smaller data regimes. 
This highlights that the batch size is not a static hyperparameter but a dynamic one that must be carefully calibrated relative to the total available data scale.

\subsection{Impact of base model capacity}

To examine how the capacity of the backbone model influences batch size effect, we compared two different models, LLaMA-2-7B and 13B. Our empirical results in Figure \ref{fig:model_shift} demonstrate that the impact of batch size remains remarkably consistent regardless of the model scale. This suggests that the relationship between batch size and the model capacity is scale invariant. Therefore, the batch size configurations identified on smaller, more manageable models can be reliably extrapolated to larger models, significantly reducing the computational cost of hyperparameter tuning.

\subsection{Guidelines for small-scale proxies}
\label{subsec:guideline}

Building on our analysis, we establish a practical strategy to identify the optimal batch size for LoRA fine-tuning. The batch size effect is scale invariant regarding LoRA rank and model capacity, while it is sensitive to the total data scale. In conclusion, the most reliable proxy is to employ a smaller model with a lower rank while training on the full target dataset. This approach ensures high fidelity hyperparameter transfer with a fraction of the original computational overhead.

\subsection {Theoretical justification}

Our finding that the optimal batch size is invariant to the base model capacity is consistent with recent theoretical findings on the infinite-width regime \cite{yang2021tensor4}. Beyond a certain model capacity threshold, the training dynamics become largely independent of the base model size, explaining the consistent batch size trends observed across different model scales \cite{Zhang2025CBS_llm}.
Furthermore, the dominance of the dataset scale over rank can be analyzed through the lens of the gradient noise scale \cite{mccandlish2018}.

To formally analyze this, consider the following setup: Let the dataset $\{(\mathbf{x}_i, y_i)\}_{i=1}^N$ be drawn i.i.d.\ from $\mathcal{N}(\mathbf{0}, \mathbf{I}_r)\times \mathrm{Unif}(\{-1, +1\})$, where the features are embedded into $\tilde{\mathbf{x}}_i$ via $\tilde{\mathbf{x}}_i := \mathbf{U}\mathbf{x}_i$ using an orthogonal matrix $\mathbf{U} \in \{\mathbf{A}\in \mathbb{R}^{d\times r} \mid \mathbf{A}^\top \mathbf{A} = \mathbf{I}_r\}$. For a quadratic objective $\hat{L}(\mathbf{w}) := \tfrac{1}{2N} \sum_{i=1}^{N}(\mathbf{w}^\top \tilde{\mathbf{x}}_i - y_i)^2$, the relationship between the dataset size $N$ and the LoRA rank $r$ can be characterized by the following lemma: 

\begin{lemma}\label{lemma1}
For $\hat{L}(\mathbf{w})$ with any rank $r > 2$, the expected GNS proxy $\mathbb{E}[B_{\mathrm{simple}}]$ satisfies:
\begin{equation}
\mathbb{E}[B_{\mathrm{simple}}] = \frac{rN}{r-2}
\end{equation}
\end{lemma}

This lemma provides a theoretical basis for our empirical observation: the dataset scale ($N$) is the primary determinant of the optimal batch size, while the influence of the LoRA rank ($r$) becomes marginal as it increases. Further discussion on the gradient noise scale and the full proof of lemma \ref{lemma1} are provided in Appendix \ref{app:theorem}.

%% file: latex/figure/shifted_acc.tex
\begin{figure*}[t!]
\centering
    \begin{subfigure}[t]{0.32\linewidth}
        \centering
        \includegraphics[width=\linewidth]{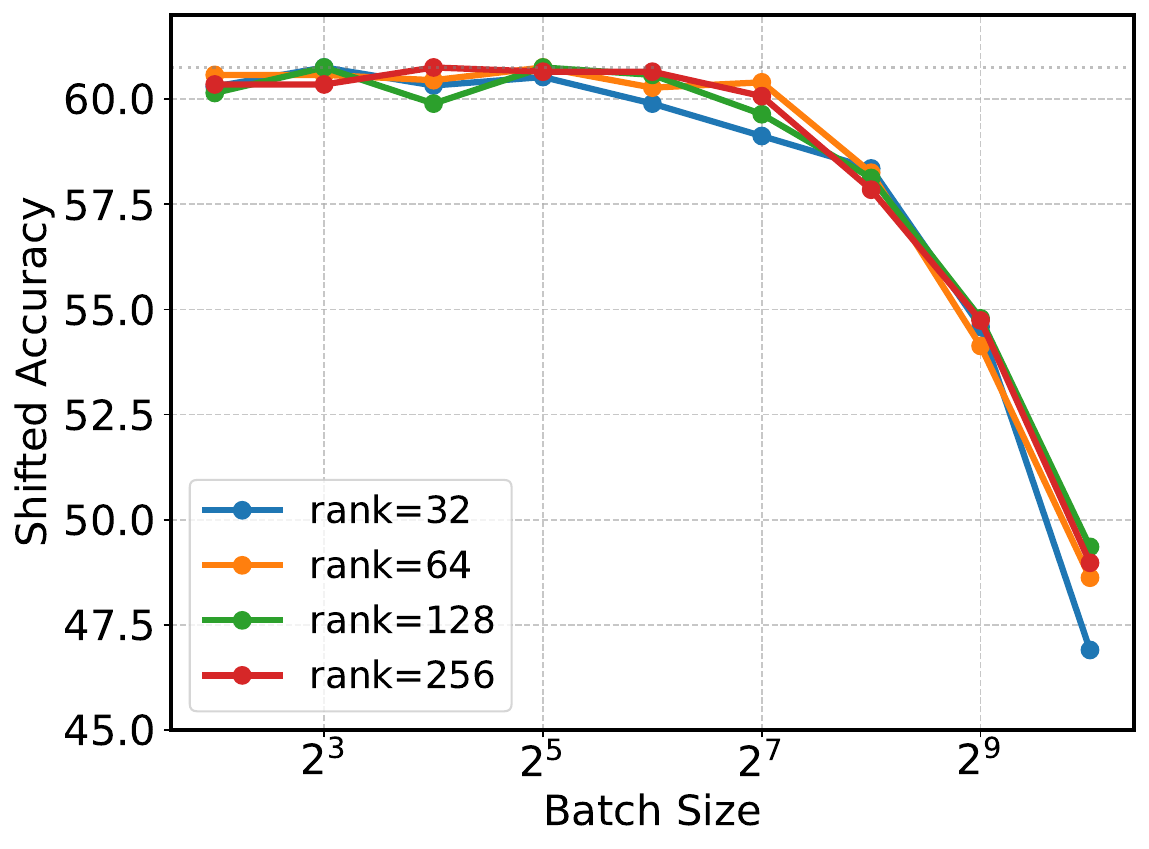}
        \captionsetup{skip=0pt}
        \caption{LoRA rank}
        \label{fig:rank_shift}
    \end{subfigure}
    \hfill
    \begin{subfigure}[t]{0.32\linewidth}
        \centering
        \includegraphics[width=\linewidth]{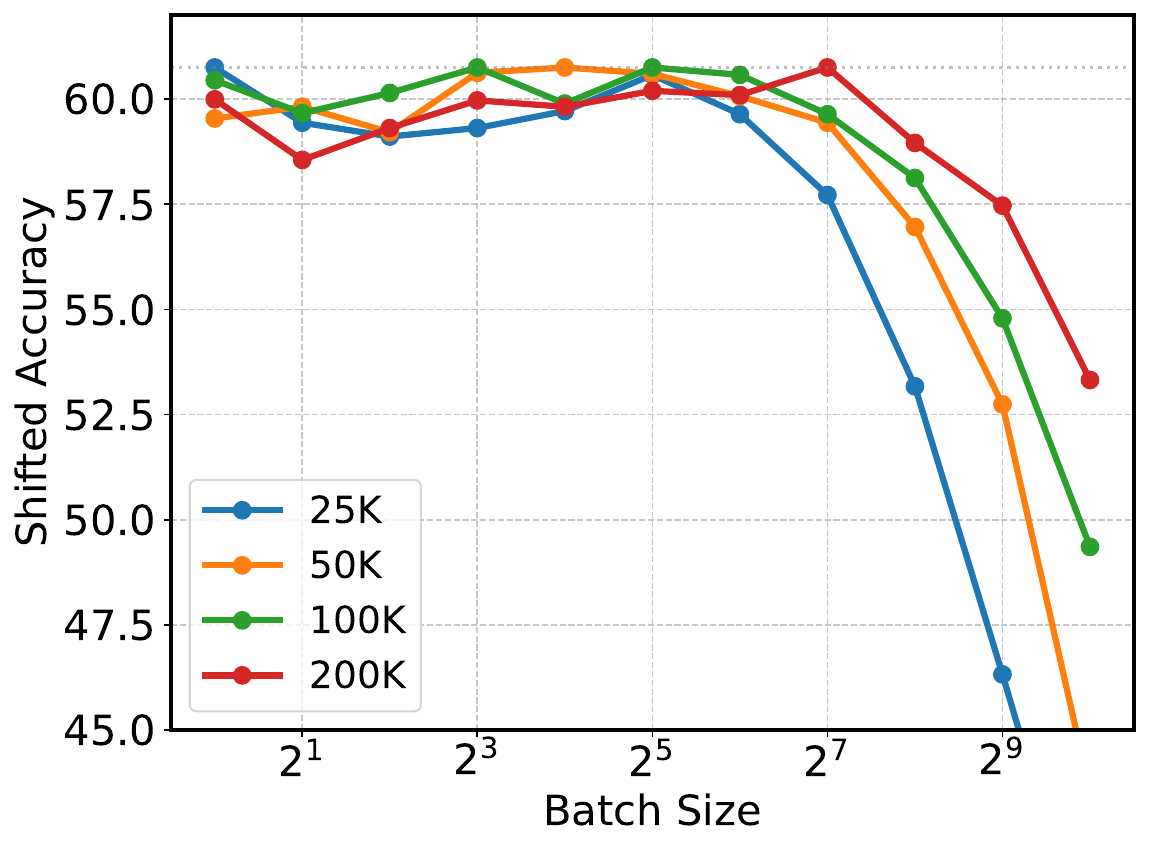}
        \captionsetup{skip=0pt}
        \caption{Dataset scale}
        \label{fig:data_shift}
    \end{subfigure}
    \hfill
    \begin{subfigure}[t]{0.32\linewidth}
        \centering
        \includegraphics[width=\linewidth]{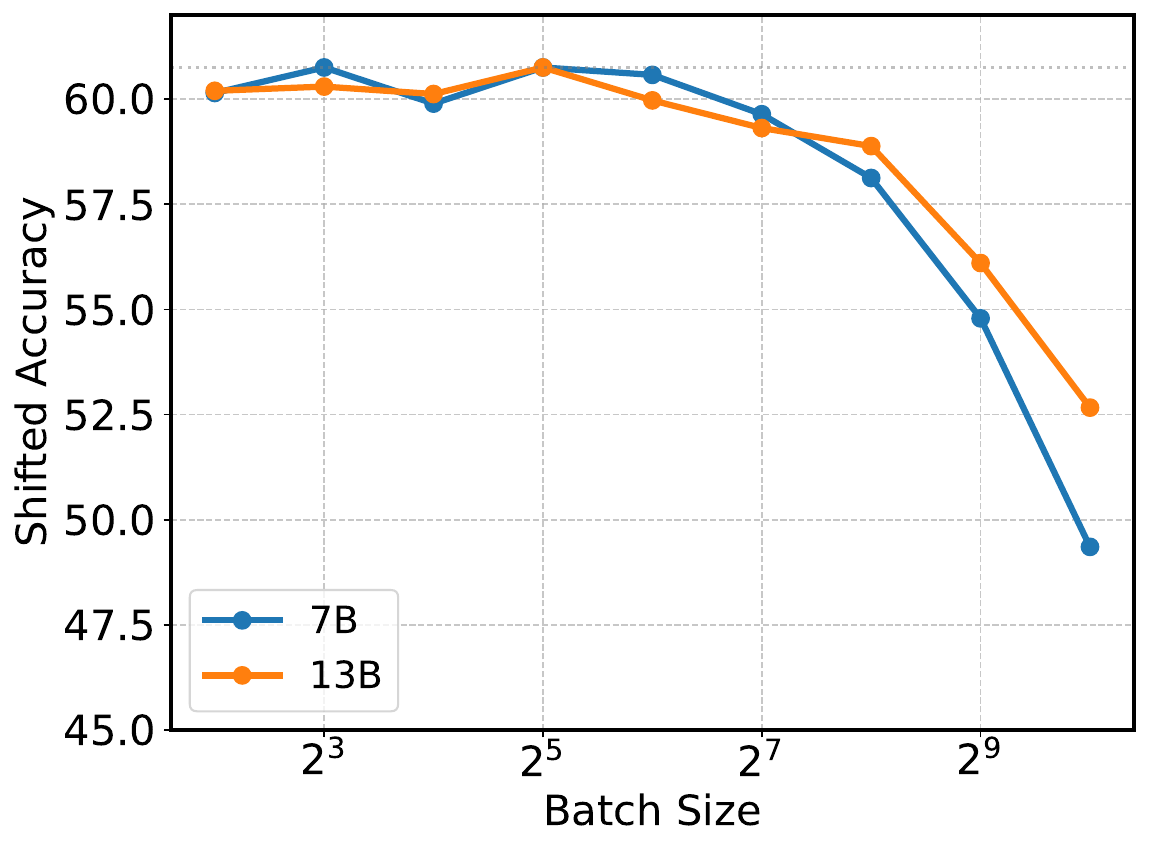}
        \captionsetup{skip=0pt}
        \caption{Base model capacity}
        \label{fig:model_shift}
    \end{subfigure}
    \hfill
\caption{\textbf{Effect of batch size across key determinants.} We examine the interaction between batch size and three factors.
For a unified comparison, accuracies are normalized by shifting the maximum value of each setup to match the default configuration (rank = 128, dataset scale = 100K, model capacity = 7B). Original accuracy values are provided in Figure \ref{fig:orr_acc} for reference.}
\label{fig:shift_acc}
\end{figure*}

%% file: latex/5_conclusion.tex
\section{Conclusion}

In this study, we provide a comprehensive re-evaluation of LoRA and its prominent variants, specifically focusing on the critical yet overlooked impact of batch size on fine-tuning performance. Our analysis demonstrates that reported advantages of some LoRA variants are confounded by hyperparameter bias. Notably, we highlight that vanilla LoRA remains a remarkably strong baseline when evaluated under optimized conditions.

Contrary to conventional heuristics suggesting that smaller batches yield lower regret, we identify the optimal point, where large batches do not harm performance. Crucially, we reveal that the optimal batch size of LoRA can be found efficiently by utilizing a smaller model with a small rank, while preserving the scale of dataset. These insights provide a rigorous foundation for batch size effects in LoRA and offer practical guidance for efficient resource utilization in limited environments.

%% file: latex/6.limit.tex
\section*{Limitations}
While our analysis provides empirical and practical insights, several avenues remain for further investigation. First, although we provide a theoretical justification via the gradient noise scale, a comprehensive theoretical framework that fully captures the non-convex optimization dynamics of LoRA is still lacking. This prevents the derivation of universal rules for optimal batch size selection. Second, while our evaluations encompass diverse model families and tasks, computational constraints limited our ability to conduct the in-depth determinant analysis across every possible task and architectural configuration. Future research is required to validate these observed patterns across even broader scales and diverse training objectives.

%% file: latex/App_related.tex
\section{Related Work}
\label{app:related_work}

\noindent\textbf{Batch size in deep learning and LLMs.} The influence of batch size on optimization and generalization has been extensively investigated in the broader deep learning literature \cite{keskar2017, hoffer2017, shallue2019, mccandlish2018, golmant2018}. \citet{mccandlish2018} characterize a critical batch size $\mathcal{B}_{crit}$ beyond which increasing the batch size yields diminishing returns in training acceleration, along with the concept of gradient noise scale. While these principles are well established for standard architectures, their application to modern Large Language Models (LLMs) is an area of active refinement. Recent work revisits these questions in the LLM training setting. In particular, \citet{Zhang2025CBS_llm} systematically measure critical batch size across model and data scales and find that it is driven primarily by data size rather than parameter count. On the fine-tuning side, \citet{Pareja2025unveil_SFT} observe that larger batches can be beneficial when paired with appropriately reduced learning rates in full fine-tuning. At the other extreme, \citet{Marek2025smallbatch} show that very small batches can train stably with batch size-aware optimizer hyperparameters, and they argue that gradient accumulation can be compute-inefficient when its primary role is to emulate large batches.

\noindent\textbf{Distinct dynamics of LoRA.} A growing body of evidence suggests that findings from full fine-tuning (FFT) do not directly translate to parameter-efficient fine-tuning methods such as LoRA. The low rank constraint imposes a distinct optimization dynamics; \citet{Shuttleworth2025} show that FFT updates typically evolve with a significantly higher effective rank than updates of the same rank LoRA. Furthermore, LoRA tends to exhibit a "learn-less and forget-less" phenomenon \cite{Biderman2025learnless}, characterized by smaller weight shifts and better preservation of pre-trained features compared to FFT. These differences imply that the sensitivity to batch size in LoRA may follow a trajectory distinct from standard patterns.

While \citet{Pareja2025unveil_SFT} suggest that larger batches are generally beneficial for FFT, recent work \cite{schulman2025lora_regret} investigates a ``low regret regime'', where LoRA can achieve similar performance to full fine-tuning, arguing that LoRA is less tolerant of large-batch training than FFT. Moreover, they suggest this gap may not matter much in practice, since smaller batches are better in both LoRA and full fine-tuning. Our work diverges from this binary view; we demonstrate that neither smaller nor larger batches are universally optimal. Instead, we identify the optimal batch size that can be determined through three fundamental determinants of LoRA rank, dataset scale, and model capacity, as detailed in Section \ref{sec:determinants}.

%% file: latex/App_steps.tex
\section{Experiments Under Fixed Steps}
\label{app:steps}

\input{latex/figure/step}

Figure \ref{fig:step} illustrates test performance across varying batch sizes under a constrained 200 optimization steps. In this configuration, where training with a batch size of 512 corresponds to approximately one epoch of the MetaMathQA100K dataset, larger batch sizes naturally entail a greater volume of training samples. Consequently, we observe a positive correlation between batch size and performance, as the increased data outweighs the potential impact of batch size. This indicates that the degradation observed in fixed sample protocol is partially attributable to insufficient updates rather than batch size itself. However, we emphasize that larger batches do not always yield superior outcomes even in a fixed step setup; rather, there is a point of diminishing returns where the effect of batch size begins to outweigh the effect of data volume

%% file: latex/figure/step.tex
\begin{figure}[t!]
\begin{center}
\includegraphics[width=0.42\textwidth]{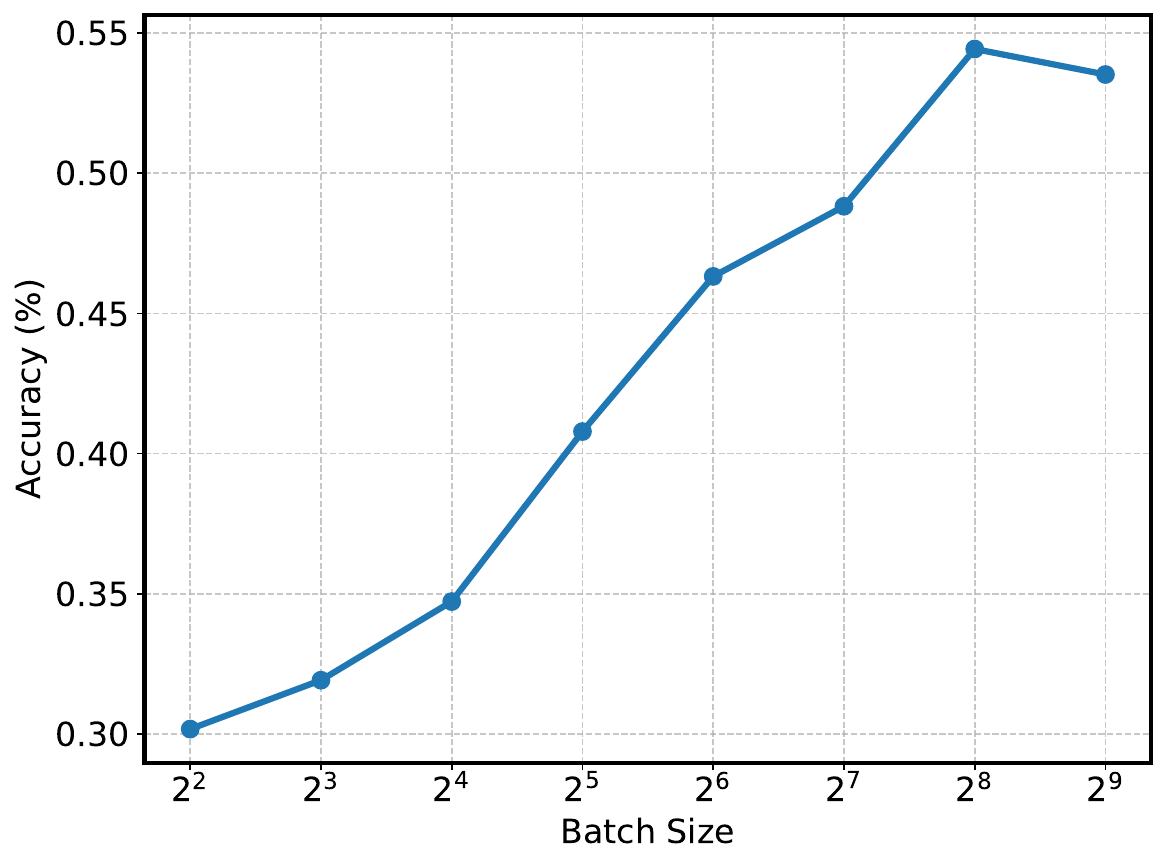}
\end{center}
   \caption{\textbf{Batch size effect under fixed training steps.} Under a fixed number of optimization steps, larger batch sizes lead to superior performance due to increased data throughput. This trend persists until a critical threshold is reached, where increasing batch size beyond which increasing batch size no longer yields improvements in test accuracy.
}
   \label{fig:step}
\end{figure}

%% file: latex/App_warmpup.tex
\section{Ablations: Warm-up \& Learning Rate Scheduling}
\label{app:warmup}

\input{latex/figure/warmup}
We set the warm-up ratio to zero for all experiments. This decision stems from the fact that varying the batch size alters the total number of training steps for a fixed sample protocol, making it difficult to adjust the warm-up period across different configurations. Furthermore, \citet{Pareja2025unveil_SFT} suggest that warm-up steps and learning rate schedulers have a negligible impact on performance in supervised fine-tuning scenarios.

In this section, we evaluate the sensitivity of LoRA fine-tuning to warm-up phases and learning rate scheduling. Figure \ref{fig:warmup} illustrates the test accuracies across a range of batch sizes under varying optimization configurations. Consistent with the observations of \citet{Pareja2025unveil_SFT}, our results indicate that the inclusion of a warm-up stage yields marginal performance differences; rather, it leads to greater robustness across a wide range of batch sizes. In contrast, we observed that the learning rate scheduler (e.g. cosine decay) itself significantly influences accuracy, resulting in a substantial performance degradation of approximately 5\% in accuracy. Consequently, we exclude only the warm-up phase from our setup.

%% file: latex/figure/warmup.tex
\begin{figure}[t!]
\begin{center}
\includegraphics[width=0.42\textwidth]{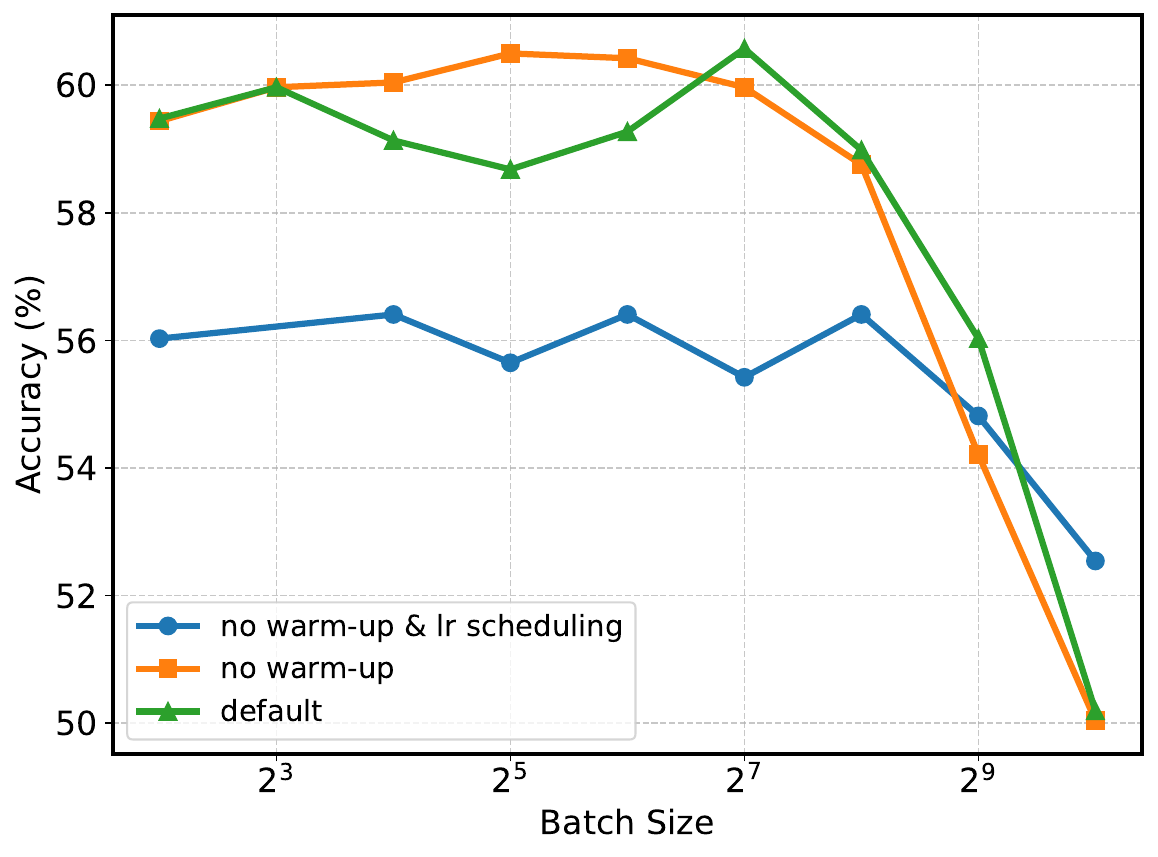}
\end{center}
   \caption{\textbf{Impact of warm-up phase and lr scheduling.} We show that while removing the warm-up phase maintains robust performance across all batch sizes without significant accuracy loss, the LR scheduling remains critical to model performance.
}
   \label{fig:warmup}
\end{figure}

%% file: latex/App_config.tex
\section{Hyperparameter Configurations}
\label{app:config}

In Table \ref{tab:config}, we provide our detailed configurations of all hyperparameters. This setup basically follows that of \citet{Meng2024pissa}, except for some settings discussed in Section \ref{sec:setup}.

\input{latex/figure/config}

%% file: latex/figure/config.tex
\begin{table}[h]
\centering
\begin{tabularx}{0.8\linewidth}{@{} >{\raggedright\arraybackslash}p{3cm} X @{}}
\toprule
\multicolumn{2}{c}{\textbf{Default Configurations}} \\
\midrule
rank $r$      & 128 \\
$\alpha$      & same as rank $r$ \\
Optimizer     & AdamW \\
Dropout       & 0 \\
LR Scheduler  & cosine \\
Warmup Ratio  & 0 \\
Epoch         & 1 \\
Placement     & query, key, value, output, gate, MLP up, MLP down \\
\bottomrule
\end{tabularx}
\caption{\textbf{Detailed hyperparameter configurations.}}
\label{tab:config}
\end{table}

%% file: latex/App_lr.tex
\section{Optimal Learning Rate Dynamics}
\label{app:opt_lr}

\input{latex/figure/lr}

As detailed in Section \ref{sec:setup}, we determine the optimal learning rate (lr) independently for each batch size configuration to ensure a fair comparison. Specifically, we conduct a grid search across an lr range of $[1\times10^{-5}, 3\times10^{-3}]$. To ensure a precise identification of the optimal point, we employ a refined logarithmic grid consisting of $\{1, 2, 5\} \times 10^n$, with additional points such as $3\times10^{-4}$ to provide higher granularity in the high performance region where the optimal lr frequently appears. We also include $3\times10^{-3}$ as the learning rate approaches the stability boundary. Figure \ref{fig:lr} illustrates the resulting optimal lr trajectory for each batch size in our default setup with a single seed.

Our analysis reveals that the optimal learning rate exhibits a non-monotonic relationship with batch size: the lr initially increases but subsequently declines as the batch size continues to increase. This behavior in LoRA scenarios aligns with recent findings by \citet{Li2024surginglr}, suggesting that the standard linear scaling rule may not hold for Adam style optimizers.

\input{latex/figure/org_acc}

%% file: latex/figure/lr.tex
\begin{figure}[t!]
\begin{center}
\includegraphics[width=0.42\textwidth]{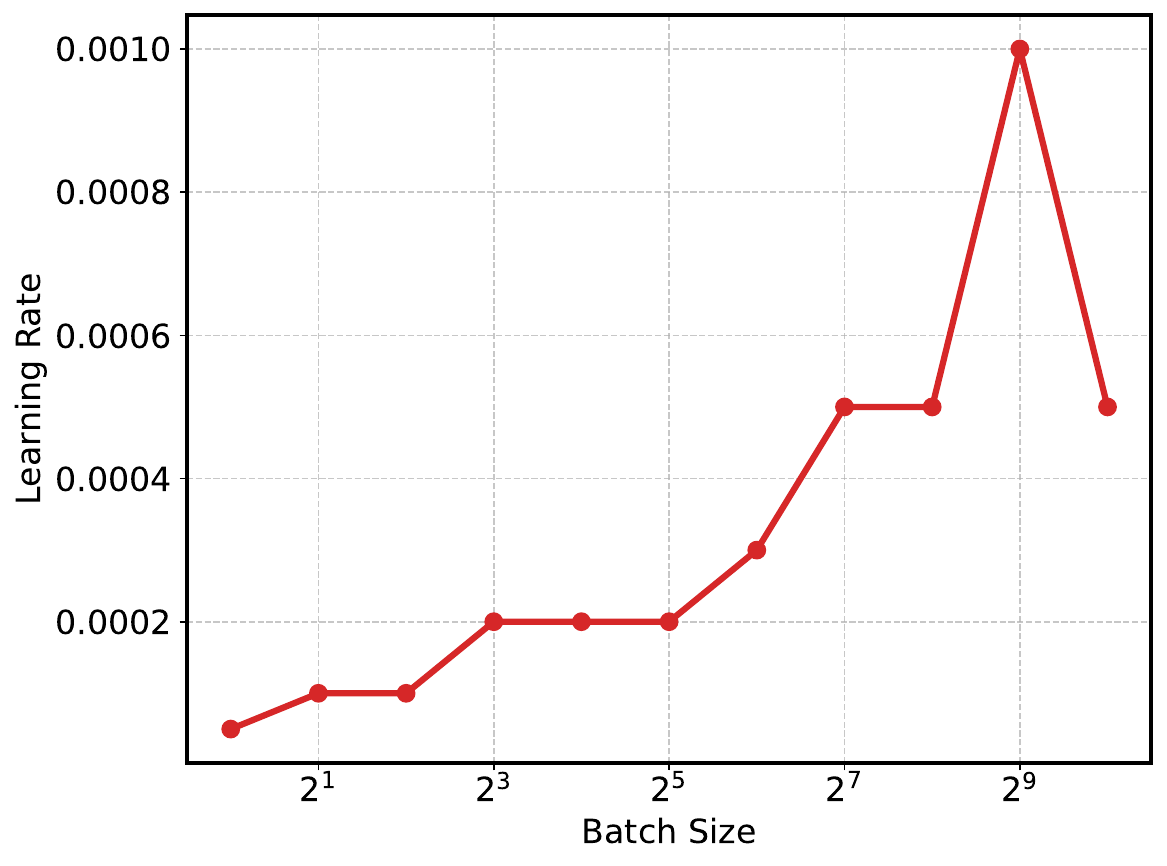}
\end{center}
   \caption{\textbf{Interaction between optimal learning rate and batch size.} We demonstrate that the optimal learning rate follows a non-monotonic trajectory as batch size increases, initially scaling upward before declining beyond a critical threshold.
}
   \label{fig:lr}
\end{figure}

%% file: latex/figure/org_acc.tex
\begin{figure*}[t!]
\centering
    \begin{subfigure}[t]{0.32\linewidth}
        \centering
        \includegraphics[width=\linewidth]{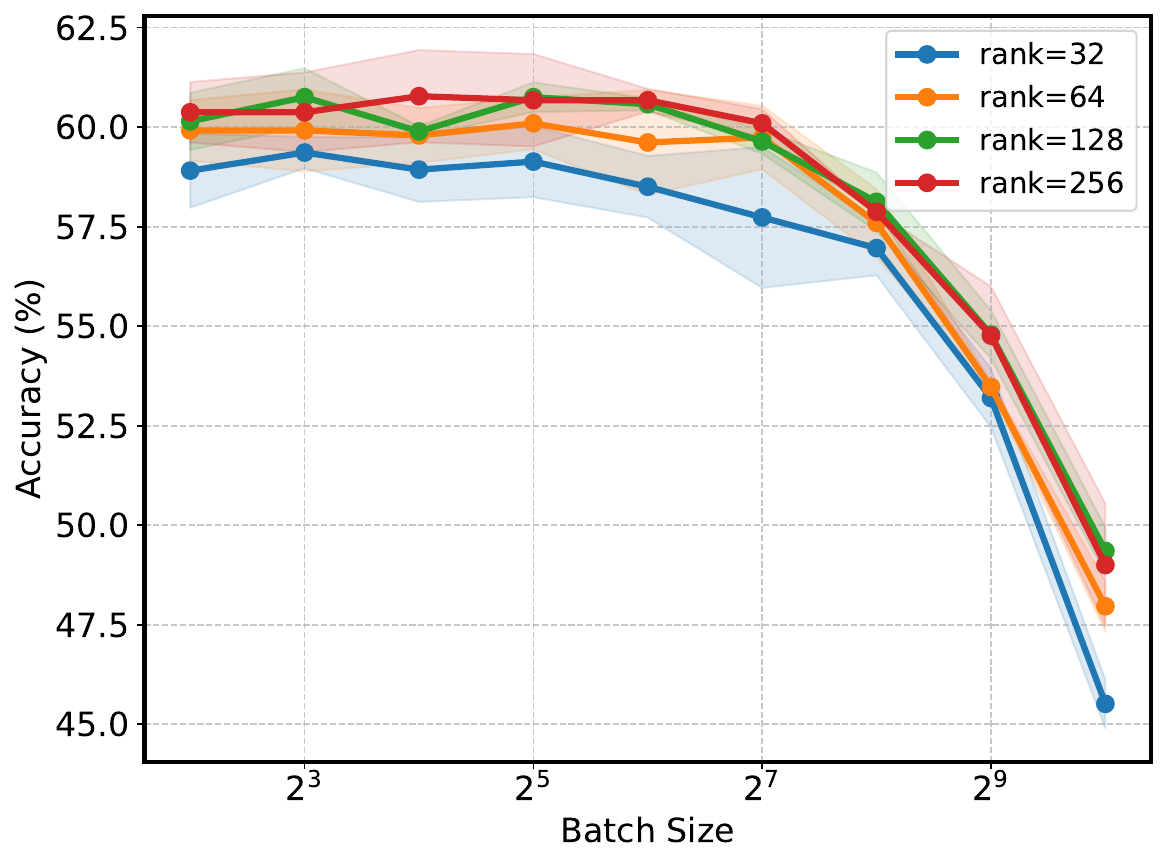}
        \captionsetup{skip=0pt}
        \caption{LoRA rank}
        \label{fig:rank_org}
    \end{subfigure}
    \hfill
    \begin{subfigure}[t]{0.32\linewidth}
        \centering
        \includegraphics[width=\linewidth]{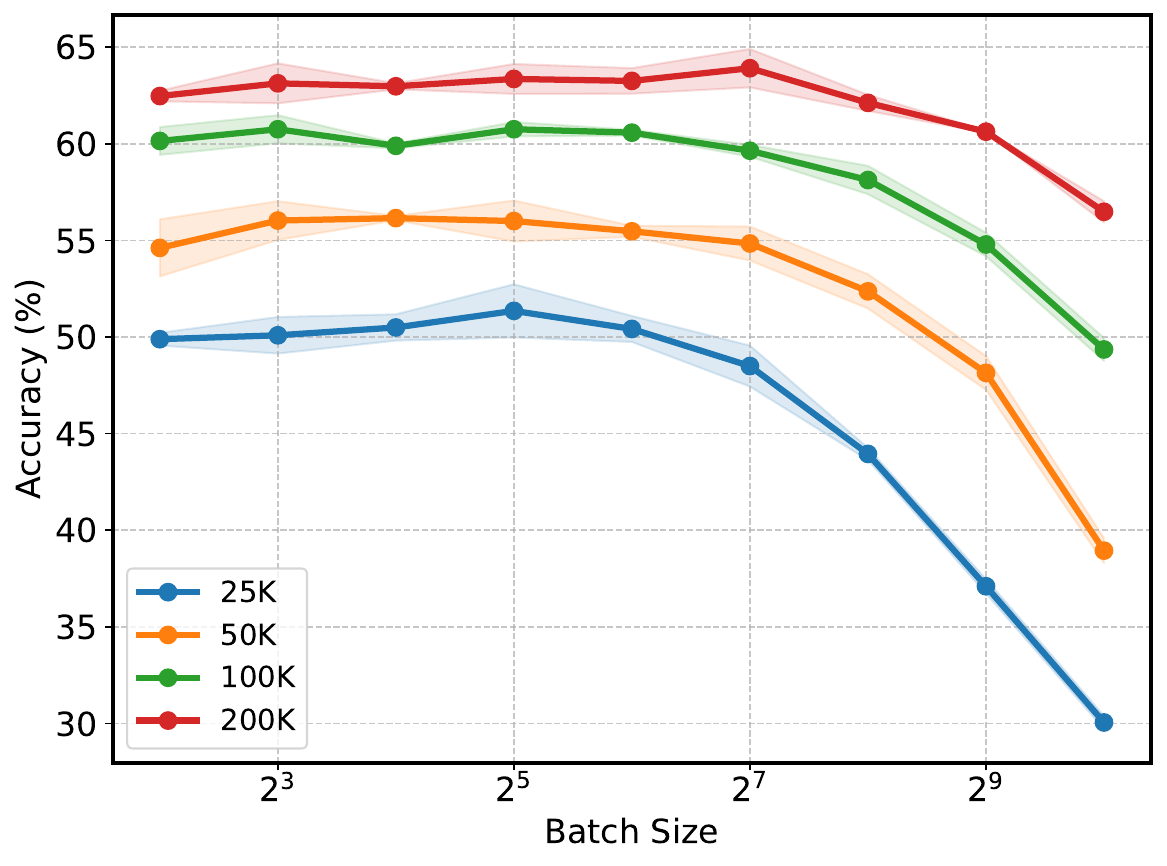}
        \captionsetup{skip=0pt}
        \caption{Dataset scale}
        \label{fig:data_org}
    \end{subfigure}
    \hfill
    \begin{subfigure}[t]{0.32\linewidth}
        \centering
        \includegraphics[width=\linewidth]{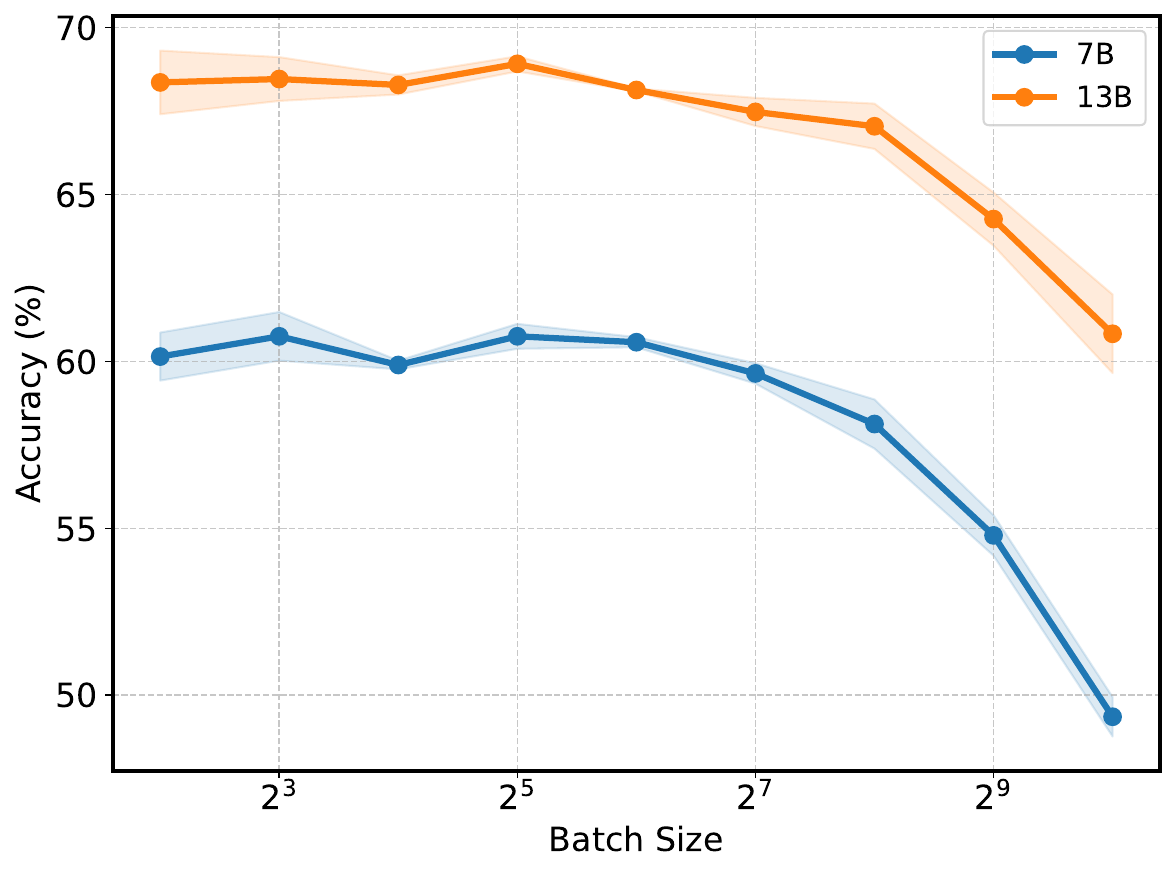}
        \captionsetup{skip=0pt}
        \caption{Base model capacity}
        \label{fig:model_org}
    \end{subfigure}
    \hfill
\caption{\textbf{Original accuracies of Figure \ref{fig:shift_acc}.} We provide original accuracies of Figure \ref{fig:shift_acc}, which are averaged over three seeds.}
\label{fig:orr_acc}
\end{figure*}

%% file: latex/App_theoretical.tex
\section{Theoretical Analysis}
\label{app:theorem}

\subsection{Gradient noise scale proxy for estimating the optimal batch size}

\citet{mccandlish2018} introduce the concept of the \textbf{gradient noise scale (GNS)} to determine the optimal batch size. Let $g$ denote the full gradient and $\Sigma$ represent the per-example gradient covariance matrix. The optimal batch size $B_{\text{crit}}$ can be estimated as the ratio of the gradient noise to the gradient magnitude, weighted by the Hessian $H$:
\begin{equation}\label{eq:B_noise}
B_{\mathrm{noise}} = \frac{\text{tr}(\Sigma H)}{g^\top H g}
\end{equation}
Under the assumption of a well-conditioned optimization landscape, \cref{eq:B_noise} simplifies to the GNS proxy as:
\begin{equation}\label{eq:B_simple}
B_{\mathrm{simple}} \approx \frac{\text{tr}(\Sigma)}{|g|^2}
\end{equation}

\subsection{Proof of Lemma \ref{lemma1}}

\begin{proof}
We derive the expected gradient noise scale at initialization ($w_0 = 0$) by analyzing the trace of the gradient covariance and the norm of the full gradient.

\paragraph{Trace of Covariance.} First, at initialization, the per-example gradient is given by $\nabla L_i(0) = -y_i \tilde{x}_i = -y_i U x_i \in \text{col}(U)$. Since all stochastic gradients reside within this $r$-dimensional subspace, the per-example gradient covariance $\Sigma$ has a rank of at most $r$. Consequently, its trace is given by $\text{tr}(\Sigma) = r$.

\paragraph{Full Gradient Norm.} The full gradient $g = \frac{1}{N} \sum \nabla L_i$ follows a Gaussian distribution $g \sim \mathcal{N}(0, \frac{1}{N} UU^\top)$. Given the orthonormality of $U$ (i.e., $U^\top U = I_r$), the squared norm follows a scaled chi-squared distribution, $\|g\|^2 \sim \frac{1}{N} \chi^2_r$.

\paragraph{Expected Gradient Noise Scale.} Finally, by leveraging the property of the inverse-chi-squared distribution, where $\mathbb{E}[1/\chi^2_r] = 1/(r-2)$ for $r > 2$, we calculate the expected noise scale as follows:
\begin{equation}
    \mathbb{E}[B_{\mathrm{simple}}] \approx \text{tr}(\Sigma) \cdot \mathbb{E}\left[\frac{1}{\|g\|^2}\right] = \frac{rN}{r-2}.
\end{equation}
\end{proof}

%% file: latex/App_ablation.tex
\section{Ablations: Model Families \& Task}
\label{app:ablation}

\input{latex/figure/ablation}

We evaluate the impact of batch size across additional model architectures, specifically Qwen3-0.6B and Gemma3-1B (see Figures \ref{fig:qwen} and \ref{fig:gemma}). Consistent with our primary findings using LLaMA-2-7B, these results highlight the existence of an optimal batch size, where increasing the batch size does not compromise accuracy. Furthermore, we provide results for the HumanEval benchmark fine-tuned on CodeFeedback100K  (Figure \ref{fig:humaneval}). Although the precise optimal threshold varies, these findings suggest that the observed batch size dynamics generalize across diverse model families and task domains.

%% file: latex/figure/ablation.tex
\begin{figure*}[t!]
\centering
    \begin{subfigure}[t]{0.32\linewidth}
        \centering
        \includegraphics[width=\linewidth]{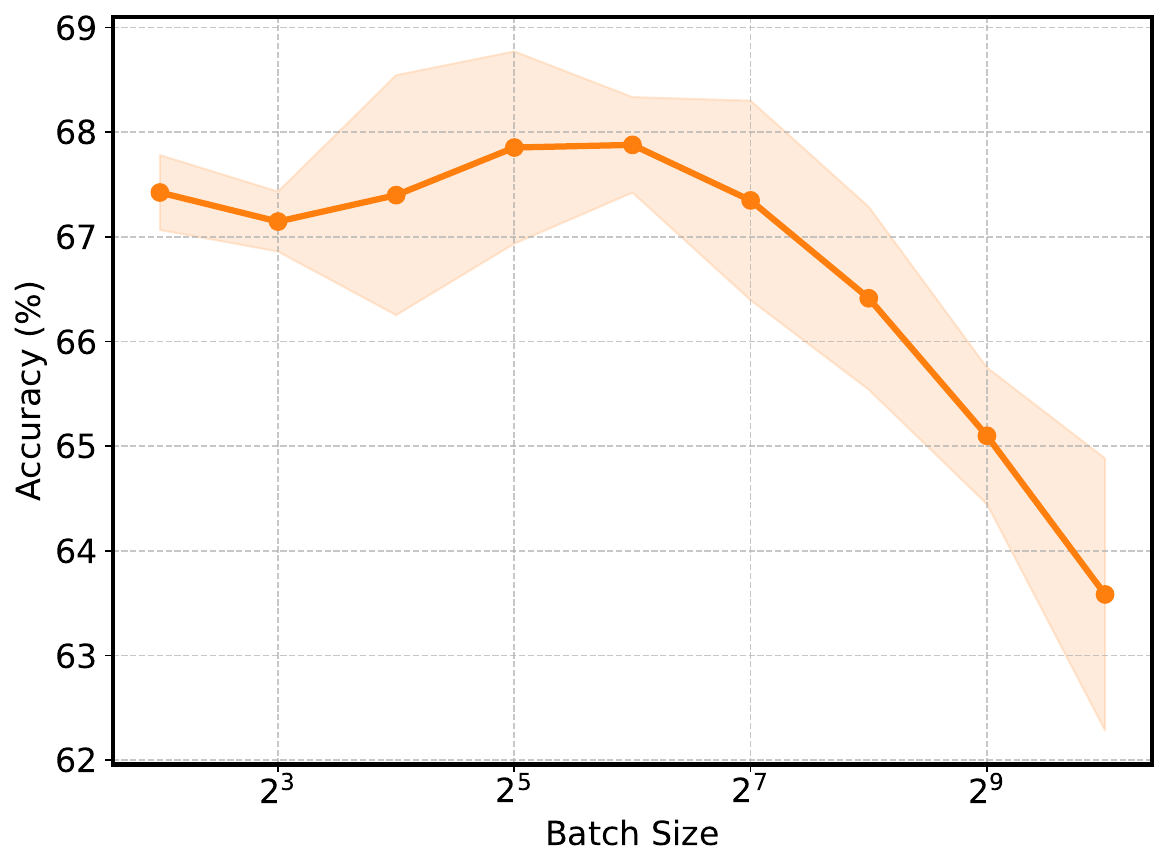}
        \captionsetup{skip=0pt}
        \caption{Qwen3-0.6B}
        \label{fig:qwen}
    \end{subfigure}
    \hfill
    \begin{subfigure}[t]{0.32\linewidth}
        \centering
        \includegraphics[width=\linewidth]{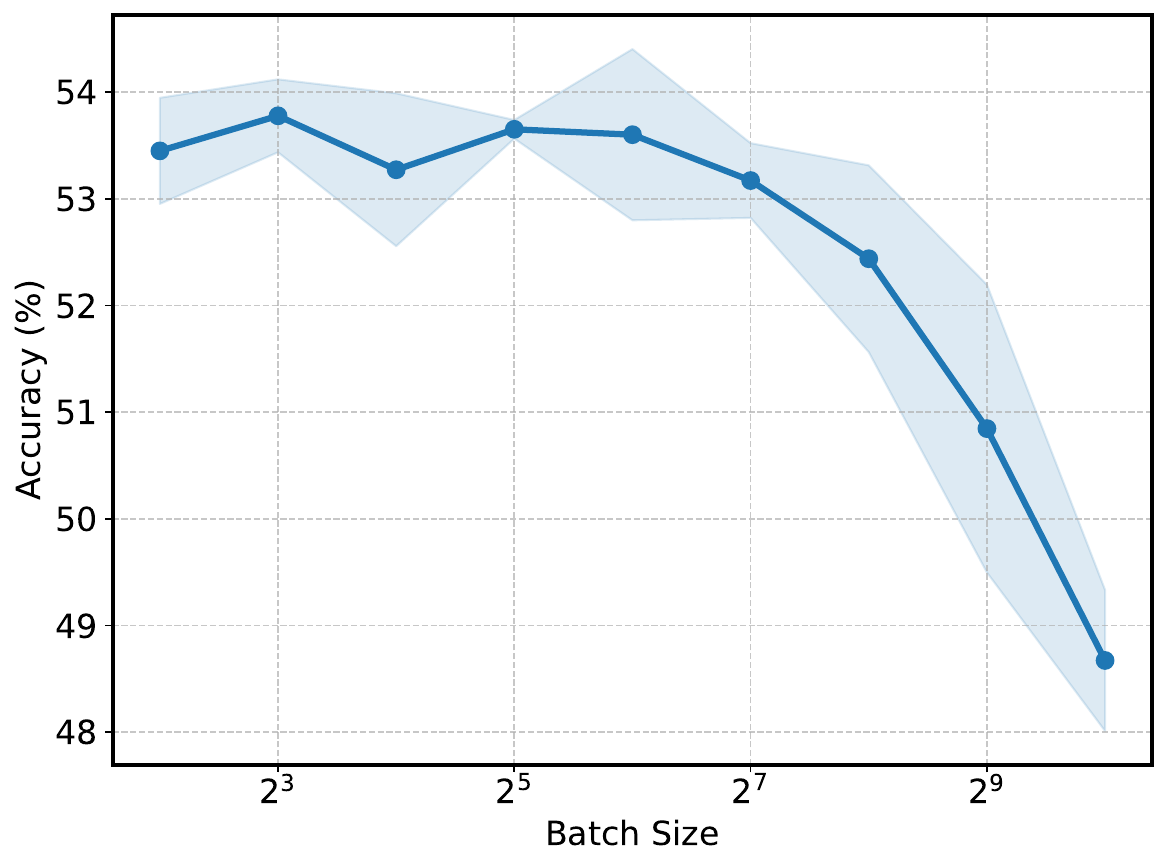}
        \captionsetup{skip=0pt}
        \caption{Gemma3-1B}
        \label{fig:gemma}
    \end{subfigure}
    \hfill
    \begin{subfigure}[t]{0.32\linewidth}
        \centering
        \includegraphics[width=\linewidth]{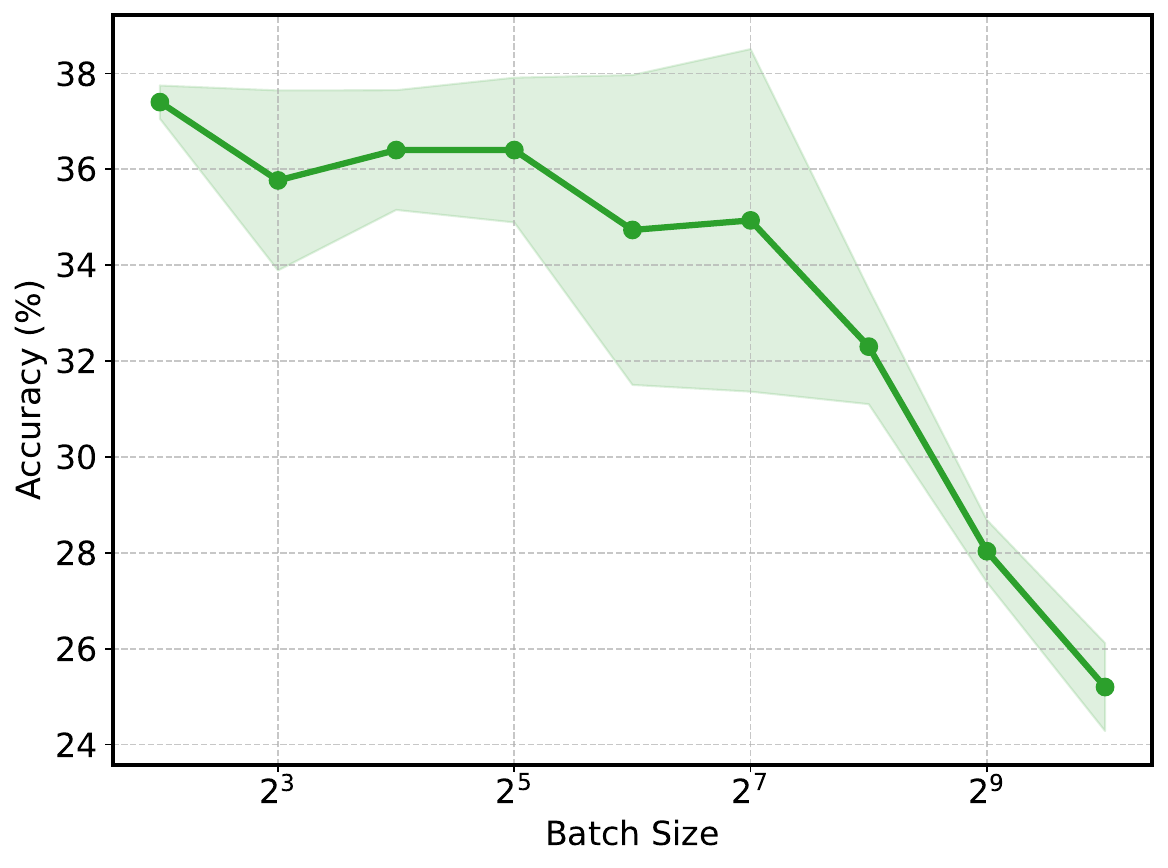}
        \captionsetup{skip=0pt}
        \caption{Code Generation}
        \label{fig:humaneval}
    \end{subfigure}
    \hfill
\caption{\textbf{Results on other model families and code generation task.} We demonstrate performances with batch size sweep for recent model families and the code generation task, which are averaged over three seeds.}
\label{fig:others}
\end{figure*}

%% file: latex/App_gpuhour.tex
\section{Discussion for Computational Scale}

To ensure empirical rigor, our experimental design relies on an exhaustive hyperparameter search and evaluation across multiple random seeds, requiring substantial computational resources. In this section, we quantify the computational cost required for our LoRA hyperparameter search, highlighting the utility of our findings to the community. 

Each single LoRA performance curve (e.g., in Figure \ref{fig:method}) requires approximately 1,058 GPU hours (~1.5 GPU months) of training on NVIDIA RTX 6000 Ada GPUs without inference and evaluation. Given that this exhaustive search was repeated across multiple experimental dimensions---including model families, model capacities, ranks, data scales, and tasks---the current computational scope represents a highly substantial contribution within an academic setting.

\paragraph{Implementation Details.} Each individual training takes an average of 1.96 hours. For the initial seed, we sweep over a broad grid of 7 learning rates (LR) per batch size to map the optimal performance region. For subsequent seeds, we narrow this search space to the 4 most promising LR configurations based on the initial seed's results. This yields an average of 5 LR grids per seed. Taken together, the total training time for a single performance curve ($T_{\text{total}}$) is calculated as follows:
$T_{\text{total}} = \mathrm{1.96\ \mathrm{hours}\times 4\ GPUs \times 5\ LR\ grids} \times 9\ \mathrm{Batch\ Sizes} \times 3\ \mathrm{Seeds} \approx 1,058\mathrm{\ GPU\ hours}$.

%% file: latex/App_concurrent.tex
\section{Comparison with \citet{lee2026learningratematters}}
\label{app:concurrent}

Concurrent with our work, \citet{lee2026learningratematters} investigate the hyperparameter bias of LoRA, focusing primarily on the role of the learning rate. While they offer a complementary conclusion to ours---vanilla LoRA frequently matches or outperforms its variants when hyperparameters are properly optimized---our work differs fundamentally in scope and practical applicability.

\paragraph{Batch Size vs. Learning Rate.} A foundational distinction lies in the scope of the hyperparameter search space. While \citet{lee2026learningratematters} provide an extensive grid search over learning rates, their batch size sweep is limited to a few selected combinations. We argue that their conclusion---\textit{the learning rate is a more critical factor than the batch size}--is fundamentally incomplete due to an overlooked joint dependency between batch size and learning rate. Moreover, as demonstrated by \citet{lee2026learningratematters}, an overwhelming majority of prior LoRA studies routinely adopt arbitrary default batch sizes. Our comprehensive joint tuning uncovers that the optimal choice of batch size cannot be compensated for by learning rate adjustments alone. When batch size is omitted from tuning, it introduces a substantial confounding bias, leading to artificial performance crossovers as illustrated in Figure \ref{fig:method}.

\paragraph{Scalable Small-Scale Proxies.} We introduce a highly scalable, cost-efficient proxy strategy grounded in the gradient noise scale. Consequently, one can reliably identify the optimal batch configuration using a low-capacity, low-rank proxy model while preserving the target dataset size, and seamlessly transfer this configuration to large-scale deployments at a fraction of the computational cost. 

In contrast, \citet{lee2026learningratematters} propose a strategy for determining hyperparameter ranges via second-order Hessian analysis at the initialization point. While theoretically insightful, this approach cannot fully account for the altering optimization trajectories during actual training. Furthermore, the standard SGD scaling rules cited by \citet{lee2026learningratematters} do not reliably transfer to Adam-style optimizers \cite{Li2024surginglr}, as we discuss in Appendix \ref{app:opt_lr}. This fundamentally limits the practical applicability of their proposed proxy.

%% file: latex/App_license.tex
\section{Artifact Licenses and Intended Use}
\label{sec:license}

In this section, we provide the licenses for all the pretrained models and datasets utilized in our study. Our use of all artifacts strictly adheres to their original licenses and is entirely consistent with their intended use for academic research and evaluation.

\subsection{Pretrained Models}
\begin{itemize}[noitemsep]
    \item \textbf{LLaMA-2 (7B, 13B).} Llama 2 Community License Agreement.
    \item \textbf{Qwen3-0.6B.} Apache license 2.0.
    \item \textbf{Gemma3-1B.} Gemma Terms of Use.
\end{itemize}

\subsection{Datasets}
\begin{itemize}[noitemsep]
    \item \textbf{MetaMathQA.} MIT.
    \item \textbf{GSM8K.} MIT.
    \item \textbf{CodeFeedback.} Apache license 2.0.
    \item \textbf{HumanEval.} MIT.
\end{itemize}